\documentclass[runningheads]{llncs}

\usepackage{eccvabbrv}

\usepackage{graphicx}
\usepackage{booktabs}

\usepackage[accsupp]{axessibility}  

\usepackage{hyperref}

\usepackage{orcidlink}
\usepackage{cleveref}
\usepackage{algorithm}
\usepackage{algorithmic}

\begin{document}

\title{Distributed Convolutional Neural Networks for Object Recognition} 


\author{Liang Sun\inst{1}\orcidlink{0000-0002-3319-1777} }


\institute{Shandong University of Science and Technology, Jinan Shandong province 250031, P.R. China 
\email{liangsun@sdust.edu.cn}\\
}

\maketitle

\begin{abstract}
  This paper proposes a novel loss function for training a distributed convolutional neural network (DisCNN) to recognize only a specific positive class. By mapping positive samples to a compact set in high-dimensional space and negative samples to Origin, the DisCNN extracts only the features of the positive class. An experiment is given to prove this. Thus, the features of the positive class are disentangled from those of the negative classes. The model has a lightweight architecture because only a few positive-class features need to be extracted. The model demonstrates excellent generalization on the test data and remains effective even for unseen classes. Finally, using DisCNN, object detection of positive samples embedded in a large and complex background is straightforward.
\end{abstract}

\section{Introduction}
\label{sec:intro}

Training convolutional neural networks(CNN) with cross-entropy loss is one of the best choices when dealing with the visual object recognition problem\cite{lecun2015deep,lecun1999object,lecun2002gradient,lecun2010convolutional}. As is well known, the underlying principle is to train a single CNN model that first encodes different classes of visual objects into a set of feature maps via the CNN's convolutional layers, and then maps these feature maps to a set of compact representations in a high-dimensional space via the FC layers. Here, feature maps are entangled together, so it is impossible to determine which feature maps are responsible for which classes, except for some feedforward visual verification methods\cite{yosinski2015understanding,zeiler2014visualizing,simonyan2013deep}. 

Meanwhile, cognitive neuroscience tells us that object recognition in the human brain originates from the two visual pathways from the retina to the visual cortex, i.e., the ventral and dorsal pathways\cite{gazzaniga2009cognitive}. Among them, the ventral pathway is the pathway for object recognition, whereas the dorsal pathway processes motion and spatial information. The ventral pathway has a hierarchical structure: the primary visual cortex extracts simple visual features, and higher visual cortices recognize objects, much like the CNN. However, distinct regions of the cortex process distinct object information: some handle face information, others handle tool information, and others handle scene information. These regions are distributed and respond only to the features of the corresponding objects. 

Therefore, training a CNN in a distributed manner that extracts only the features of a specific positive class and neglects all others is of significant research interest. Firstly, a headless CNN is obtained by removing the softmax layer of a classic CNN, and then a novel loss function that constrains the output of this headless CNN for a specific positive class to a compact set in a high-dimensional space is proposed, and those of the negative classes with no similar features to the positive class are constrained to Origin simultaneously. The proposed loss function is called the negative-to- Origin (N2O) loss, which is equivalent to the cross-entropy loss plus an additional constraint that forces negative samples to map to Origin. A headless CNN trained with the N2O loss is referred to as a DisCNN. An experiment is conducted to prove that the trained DisCNN's convolutional layers extract only features of the positive class. Thus, positive class features are disentangled from those of the other classes in a distributed manner, the same as object recognition in the ventral pathway of the brain. Furthermore, since the model only extracts features from the positive class, the number of feature maps output by the convolutional layers could be limited to a few, or even just one, rather than 512 or more, as in classical multi-class classification problems. Therefore, the model can be very lightweight. The model generalizes well to the test data. Moreover, when testing the model on samples from unseen classes, those with no similar features to the positive class are mapped to Origin, whereas those with similar features are mapped to the same compact set as the positive class. Finally, DisCNN can also be applied to object detection. First, a large image with positive samples embedded in complex backgrounds is evenly partitioned into patches of appropriate size, and then the patches containing a positive sample can be identified using DisCNN. 

\section{Model, Dataset, and Loss Function}
\label{sec:s2}
\subsection{Model Architecture}

DisCNN comprises four convolutional layers, with all convolution kernels set to 3x3 (conv3) as in VGG \cite{simonyan2014very}. Since the goal was merely to learn to extract the abstract features of a specific positive class, the number of such features could be only a few. For example, features in a car might include body, wheels, windows, and lights, etc., so eight feature maps may be sufficient for the output of the convolutional layer of DisCNN, also denote it as DisCNN-8. The output of the convolutional layer can be even set to 1, indicating that the entire class is only represented by a single high-level abstract feature, denoted DisCNN-1. Except for the absence of a softmax layer, the FC layers of DisCNN differ in no way from those of a classic CNN. In summary, the model comprises 7 hidden layers (4 convolutional and 3 FC layers), and its architecture is shown in \cref{tab:t1}. The VGG model for the 10-class classification problem is shown in the last column of the table for comparison, with 128 feature maps, approximately 10 for per class.

\begin{table}[tb]
  \caption{Model Architectures. The convolutional layer parameters are denoted as “conv<receptive field size>-<number of channels>”. The batch normalization are denoeted as BN.
  }
  \label{tab:t1}
  \centering
  \tabcolsep=0.5cm
  \begin{tabular}{@{}lll@{}}
    \toprule
    DisCNN-8 & DisCNN-1 & VGG\\
    \midrule
    input(96x96RGB) & input(96x96RGB) & input(96x96RGB)\\
    \midrule
    Conv3-64& Conv3-64 & Conv3-64\\
    BN & BN & BN\\
    ReLU & ReLU & ReLU\\
    Max pooling & Max pooling &Max pooling\\
    \midrule
    Conv3-32& Conv3-32 & Conv3-128\\
    BN & BN & BN\\
    ReLU & ReLU & ReLU\\
    Max pooling & Max pooling &Max pooling\\
    \midrule
    Conv3-16& Conv3-16 & Conv3-128\\
    BN & BN & BN\\
    ReLU & ReLU & ReLU\\
    Max pooling & Max pooling &Max pooling\\
    \midrule
    Conv3-8& Conv3-1 & Conv3-128\\
    BN & BN & BN\\
    ReLU & ReLU & ReLU\\
    Max pooling & Max pooling &Max pooling\\
    \midrule
    FC-288 & FC-36 & FC-4096\\
    \midrule
    FC-128 & FC-32 & FC-2048\\
    \midrule
    FC-16 & FC-16 & FC-1000\\
    \midrule
    && Softmax\\
  \bottomrule
  \end{tabular}
\end{table}

Obviously, DisCNNs have significantly fewer parameters compared to the 10-class VGG. As shown in \cref{tab:t2}, DisCNN-8 has 0.149 million parameters, much fewer than VGG. So DisCNN is lightweight.

\begin{table}[tb]
  \caption{Number of parameters (in millions).
  }
  \label{tab:t2}
  \centering
  \tabcolsep=0.5cm
  \begin{tabular}{@{}llll@{}}
    \toprule
    Network & DisCNN-8 & DisCNN-1 & VGG\\
    \midrule
    Number of parameters & 0.149 & 0.028 & 3.096\\

  \bottomrule
  \end{tabular}
\end{table}

\begin{remark}
\label{r1}
Batch Normalization is essential for the model, as its absence would directly lead to non-convergence during training.

\end{remark}

\subsection{Dataset}
STL-10 dataset\cite{coates2011analysis} is chosen to train our model. The STL-10 dataset comprises 10 classes: airplane, bird, car, cat, deer, dog, horse, monkey, ship, and truck. Each class has 500 training images and 800 test images. All of these images are from ImageNet\cite{deng2009imagenet} and have a resolution of 96x96 pixels.
These 10 classes can be aggregated into two larger classes: vehicles and mammals. Objects within the larger classes are similar, whereas objects between them do not share similar features. Based on this understanding, the car, one of the vehicles, is taken as the positive class, and \{bird, cat\}, some classes of mammals, as the negative classes. Then, DisCNN can be trained on this sub-dataset to extract only the common car features.
\begin{remark}
\label{r2}
DisCNN requires common features of the positive class and no similar features between the positive class and the negative classes in the training dataset.

\end{remark}

\begin{remark}
\label{r3}
DisCNN performs better after training with two or more negative classes than with a single negative class.
\end{remark}

\subsection{Loss Function}
The principle of the proposed N2O loss is to add a new constraint to the cross-entropy loss, which maps the negative samples to Origin. Using N2O, independent DisCNNs can be trained with respect to a positive class, such as car, ship, cat, horse, etc. In each DisCNN, the features of the positive classes are disentangled from those of the negative classes; that is, each DisCNN detects only the features of its own positive samples, and the samples from the negative classes receive zero response from it.

Although DisCNN trained with N2O loss is fairly effective, the underlying theoretical principle of N2O remains incompletely understood by me. In addition, it has not been protected by any patent, so no details of N2O are provided in this paper. Nevertheless, a trained DisCNN is provided as supplementary material for the readers to test.

\section{Main Results}

\subsection{DisCNN Extracts Only the Features of the Positive Class}
\begin{theorem}
\label{Th1}
The feature maps produced by the convolutional layers of DisCNN contain only features of the positive class, whereas those of the negative classes are neglected.
\end{theorem}

\textbf{Proof:} (in case of DisCNN-1)

First, \cref{alg:alg1} is proposed to prove the theorem.
According to \cref{alg:alg1}, two experiments are conducted on the datasets \{car, bird\} and \{cat, bird\}, respectively. When training the grafted DisCNN’-1 with \{car, bird\} as in step 3, the error converges to zero. This indicates that features of cars, or birds, or both are extracted to discriminate between them. And when using \{cat, bird\}, the error fluctuates around that of the initial parameters and has no sign of convergence. This indicates that the features of the cats and birds have not been extracted. Based on the two facts above, it can be concluded that DisCNN-1 extracts only car features and neglects those of cats and birds. 

The theorem is proved.

\begin{algorithm}
	\caption{Feature determination}
    \label{alg:alg1}
	\begin{algorithmic}
		\STATE \textbf{Step 1:} By using N2O loss, a DisCNN-1 as shown in Table 1 is trained through backpropagation, with \{Car\} as the positive class and \{cat, bird\} as the negative classes.
		\STATE \textbf{Step 2:} Initialize a new DisCNN-1 appended with a softmax layer with two nodes corresponding to the number of classes, i.e., 2, and denote it DisCNN’-1.
		\STATE \textbf{Step 3:} Graft the convolutional layer parameters of DisCNN-1 strictly to DisCNN’-1, denoted as grafted DisCNN’-1, and then train this grafted DisCNN’-1 by cross-entropy loss with a dataset consisting of two classes. During training, the parameters of the convolutional layers must be frozen and only the FC layer parameters updated.
        \STATE \textbf{Step 4:} If the training in step 3 converges, indicate that the features of at least one of the two classes are extracted by DisCNN-1; and if it does not, indicate that no features of the two classes are extracted by DisCNN-1. 
	\end{algorithmic}
\end{algorithm}

\begin{remark}
\label{r4}
According to \cref{Th1}, compared with cross-entropy, DisCNN achieves feature disentanglement, that is, the features of the positive class are disentangled with those of other contrast negative classes.

\end{remark}

\subsection{Model Evaluation}
\label{3.2}
Using the proposed N2O loss, with \{car\} as the positive class and \{bird, cat\} as the negative class, a DisCNN-8 is trained by backpropagation.

\emph{First}, the STL-10 test sub-dataset \{car, bird, cat\} is used to evaluate DisCNN-8. If the modulus of the output vectors is greater than zero, it is classified as a positive class (i.e., a car); otherwise, it is classified as a negative class. The model exhibits good generalization, as indicated by the confusion matrix in \cref{tab:t3}.

\begin{table}[tb]
  \caption{Confusion matrix
  }
  \label{tab:t3}
  \centering
  \tabcolsep=0.5cm
  \begin{tabular}{@{}lll@{}}
    \toprule
    &Positive &Negative\\
    \midrule
    True &TP = 787 &FN = 175\\
    False &FP = 13 &TN = 1425\\
  \bottomrule
  \end{tabular}
\end{table}

\begin{remark}
\label{r5}
When the distribution of the test dataset differs from that of the training dataset, some negative class encodings will deviate from Origin and go close to the positive class encodings. However, the module of the encoding vector can still be used to classify them correctly. The smaller the module, the closer the vector is to Origin, indicating a higher possibility of being a negative class. Therefore, by comparing a suitable threshold with the sample encoding module, the model can remain effective.    
\end{remark}

Thresholds can be chosen according to practical needs. If a lower false positive rate for the positive class is desired, a smaller threshold should be chosen. Conversely, if a lower false-negative rate is required for the negative class, a larger threshold should be selected. The evaluation metrics under different thresholds are shown in Table 4. According to \cref{tab:t4}, a threshold of 1 is preferable to 0. 

\begin{table}[tb]
  \caption{Evaluation Metrics of different thresholds
  }
  \label{tab:t4}
  \centering
  \tabcolsep=0.5cm
  \begin{tabular}{@{}lllll@{}}
    \toprule
    threshold&&Precision&Recall&F1\\
    \midrule
    0&Positive&0.985&0.816&0.892\\
    &Negative&0.889&0.992&0.937\\
    
    \midrule
    \textbf{1}&Positive&\textbf{0.968}&\textbf{0.924}&\textbf{0.945}\\
    &Negative&\textbf{0.960}&\textbf{0.983}&\textbf{0.972}\\
    
    \midrule
    2&Positive&0.941&0.964&0.953\\
    &Negative&0.983&0.971&0.977\\
  \bottomrule
  \end{tabular}
\end{table}

\emph{Secondly}, this model demonstrates excellent generalization to unseen classes. Unseen samples without any positive class features are encoded as zero vectors, whereas those with similar features are encoded in the same compact set as the positive class. Testing the model on \{chuck, deer, monkey\}, chuck is encoded the same way as car, and deer and monkey are encoded into zero vectors. Choosing a threshold of 1 produces fairly good results, as shown in \cref{tab:t5}.

\begin{table}[tb]
  \caption{Evaluation Metrics of different thresholds for unseen classes
  }
  \label{tab:t5}
  \centering
  \tabcolsep=0.5cm
  \begin{tabular}{@{}lllll@{}}
    \toprule
    threshold&&Precision&Recall&F1\\
    \midrule
    0&Positive&0.965&0.778&0.862\\
    &Negative&0.863&0.980&0.918\\
   
    \midrule
    \textbf{1}&Positive&\textbf{0.931}&\textbf{0.894}&\textbf{0.912}\\
    &Negative&\textbf{0.945}&\textbf{0.965}&\textbf{0.955}\\
   
    \midrule
    2&Positive&0.886&0.944&0.914\\
    &Negative&0.974&0.945&0.959\\
  \bottomrule
  \end{tabular}
\end{table}

\subsection{Object Detection with DisCNN}
To detect cars embedded in a complex background, the large image is first evenly partitioned into patches of appropriate size, and then the trained DisCNN-8 from \cref{3.2} is applied to determine whether each patch is a car. The results are shown in \cref{fig:f1}.

\begin{figure}[tb]
  \centering
  \includegraphics[height=6cm]{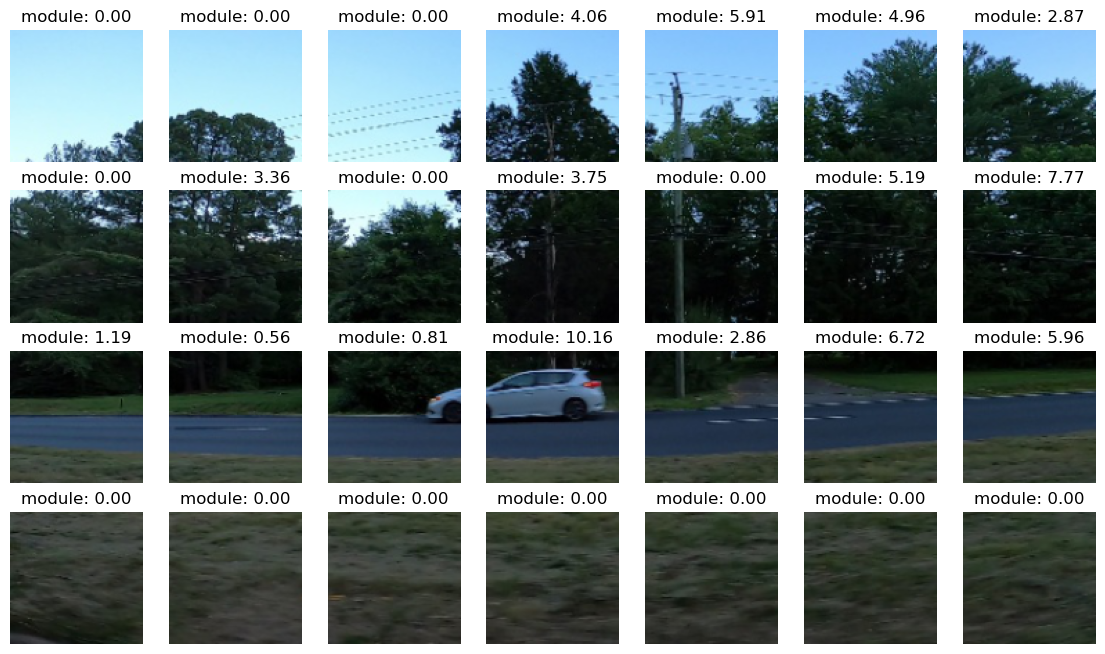}
  \caption{Patches and modules
  }
  \label{fig:f1}
\end{figure}

Since the ratio of positive to negative samples here is 1:28, quite different from the training data, a considerable portion of the background patches can activate the model, as shown in \cref{fig:f1}. So, an appropriate threshold shall be chosen according to \cref{r5}. By sorting the modules in descending order, as shown in \cref{fig:f2}, the threshold can be set to 8 to silence all background patches and locate the car patch, i.e., the third row and the fourth column.

\begin{figure}[tb]
  \centering
  \includegraphics[height=1.5cm]{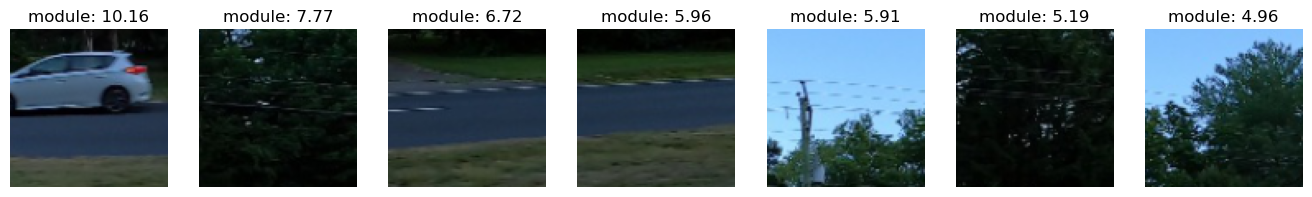}
  \caption{Patches sorted by modules
  }
  \label{fig:f2}
\end{figure}

\begin{remark}
\label{r6}
The background patches having no similar features of the positive class car also cannot activate the model, just like deer and monkey, which is quite useful for object detection.

\end{remark}

\section{Conclusions}

This paper presents DisCNN and its training loss function, which provides a distributed framework for CNN-based object recognition. The key point is that a trained DisCNN has been shown to extract only relevant positive class features. Therefore, during testing, the model will not respond to the negative classes used in training or to other unseen negative classes that share no common features with the positive class. If it is an object similar to the positive class, it will be encoded in the same compact set as the positive class, with a probability equivalent to the degree of similarity. Our model is consistent with the object-recognition architecture of the ventral pathway in the human brain, since both are distributed. As a fundamental visual representation framework, DisCNN has broad application prospects and could be a great help to existing top computer vision technologies, such as YOLO\cite{redmon2016you} for object detection, world model of spatial intelligence, and JEPA\cite{assran2023self}.

\section*{Acknowledgements}
Words cannot express my gratitude to Professor Hung-yi Lee of National Taiwan University for his outstanding online classes that gave me a comprehensive understanding of neural networks, to Professor Bin Li of South China University of Technology for his invaluable feedback that corrected a misconception I had about neural networks, without which this paper would never show up, and to Professor Yann LeCun of New York University for his invaluable public speakings that always led me in the right direction of research. I am also very grateful to Pytorch and ImageNet; these tools make everything possible.

%
%
\bibliographystyle{splncs04}
\bibliography{main}
\end{document}